\documentclass[final]{siamltexmm}

\usepackage{graphicx}
\usepackage{float}
\usepackage{amsmath}
\usepackage{amssymb}
\usepackage{framed}
\usepackage{mathdots}
\usepackage[ruled,vlined]{algorithm2e}
\usepackage{subcaption}
\usepackage{picins}
\usepackage{wrapfig}
\usepackage{pmat}

\newcommand{\eg}{{\it e.g.}}
\newcommand{\ie}{{\it i.e.}}

\newcommand{\tr}{{\mathrm {tr}}}
\newcommand{\norm}[1]{\left\Vert#1\right\Vert}

\newcommand{\set}[1]{\left\{#1\right\}}
\newcommand{\parr}[1]{\left (#1\right )}
\newcommand{\brac}[1]{\left [#1\right ]}
\newcommand{\ip}[1]{\left \langle #1 \right \rangle }
\newcommand{\Real}{\mathbb R}
\newcommand{\Complex}{\mathbb C}
\newcommand{\Natural}{\mathbb N}
\newcommand{\Exp}[1]{e^{#1}}

\DeclareMathOperator{\find}{find}
\DeclareMathOperator{\maximize}{maximize}
\DeclareMathOperator{\minimize}{minimize}
\DeclareMathOperator{\st}{subject\ to}
\DeclareMathOperator{\Mod}{mod}

\newcommand{\1}{\mathbf 1}
\newcommand{\V}{\overline{V}}
\newcommand{\PP}{\mathcal{P}}

\newcommand{\R}[1]{R\left (#1\right )}

\begin{document}

%%%%%%%%% TITLE
\title{A Global Approach for Solving Edge-Matching Puzzles}

\author{S.Z.~Kovalsky, D.~Glasner and R.~Basri}
\date{}

\maketitle
\newcommand{\slugmaster}{}

\begin{abstract}
We consider apictorial edge-matching puzzles, in which the goal is to arrange a collection of puzzle pieces with colored edges so that the colors match along the edges of adjacent pieces. We devise an algebraic representation for this problem and provide conditions under which it exactly characterizes a puzzle. Using the new representation, we recast the combinatorial, discrete problem of solving puzzles as a global, polynomial system of equations with continuous variables. We further propose new algorithms for generating approximate solutions to the continuous problem by solving a sequence of convex relaxations.
\end{abstract}

\begin{AMS}\end{AMS}

\section{Introduction}
Jigsaw puzzles \cite{AnneD.Williams2004}, dating back to the 1760s, are among the most popular single-player puzzles. The edge-matching puzzle, introduced in the 1890s, is a variation of the jigsaw puzzle in which the goal is to arrange a given collection of tiles with colored edges so that the colors match up along the edges of adjacent tiles. An example of an edge-matching puzzle with 36 square pieces is shown in Figure~\ref{fig:example_6x6}. Edge-matching puzzles are challenging compared to standard jigsaw puzzles as only an entire solution guarantees the correctness of any local part of the solution.

Despite recent breakthroughs in algorithmic solutions for pictorial puzzles \cite{Pomeranz2011, Gallagher2012, kilho_eccv2014}, in which one aims to reorganize a scrambled image, relatively little attention had been given to apictorial edge-matching puzzles.
This NP-hard problem~\cite{Demaine2001} sparked a lot of interest following the launch of the Eternity puzzle challenges. The challenge posed in the \emph{``Eternity I"} puzzle was to tile a large dodecagon with 209 irregularly shaped smaller polygonal pieces; it was marketed as being practically unsolvable, but was solved within a year, for a prize of 1 million pounds. For the \emph{``Eternity II"} puzzle one must correctly place 256 square pieces, whose edges are marked with different patterns, into a 16$\times$16 grid. The puzzle was launched in 2007 and to date remains unsolved. A prize of 2 million dollars was on offer up until December 2010.

\begin{figure}[t]
\centering
\begin{subfigure}[b]{0.49\textwidth}
  \centering
  \includegraphics[width=0.7\columnwidth]{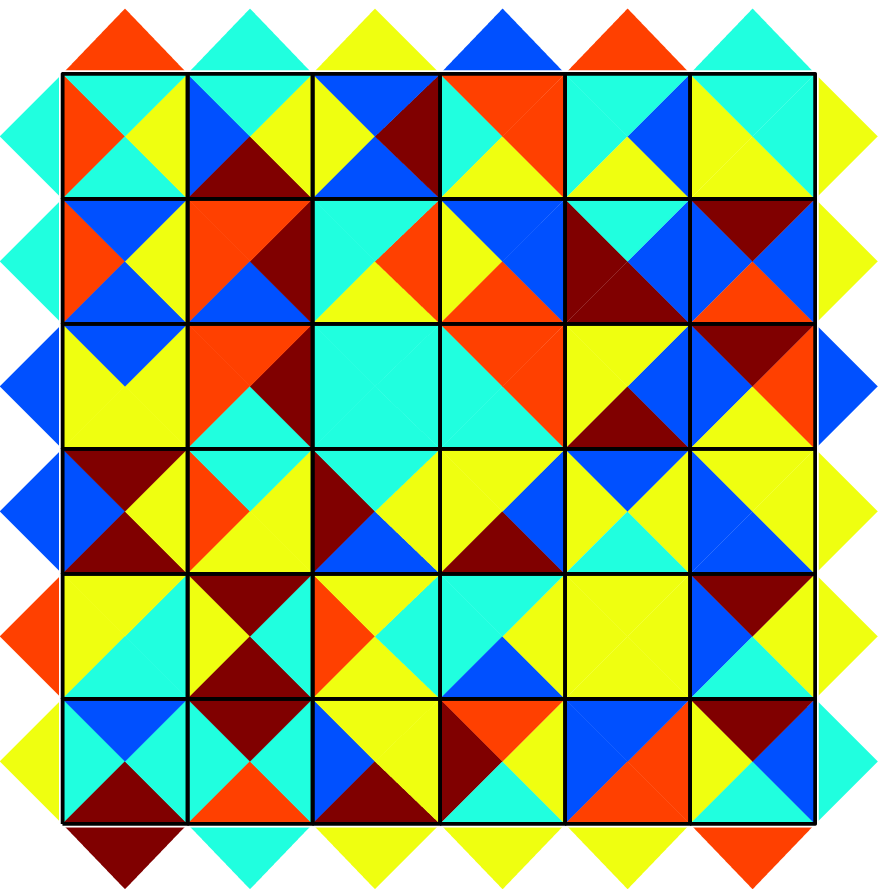}
  \caption{Scrambled}
\end{subfigure}
\begin{subfigure}[b]{0.49\columnwidth}
  \centering
  \includegraphics[width=0.7\columnwidth]{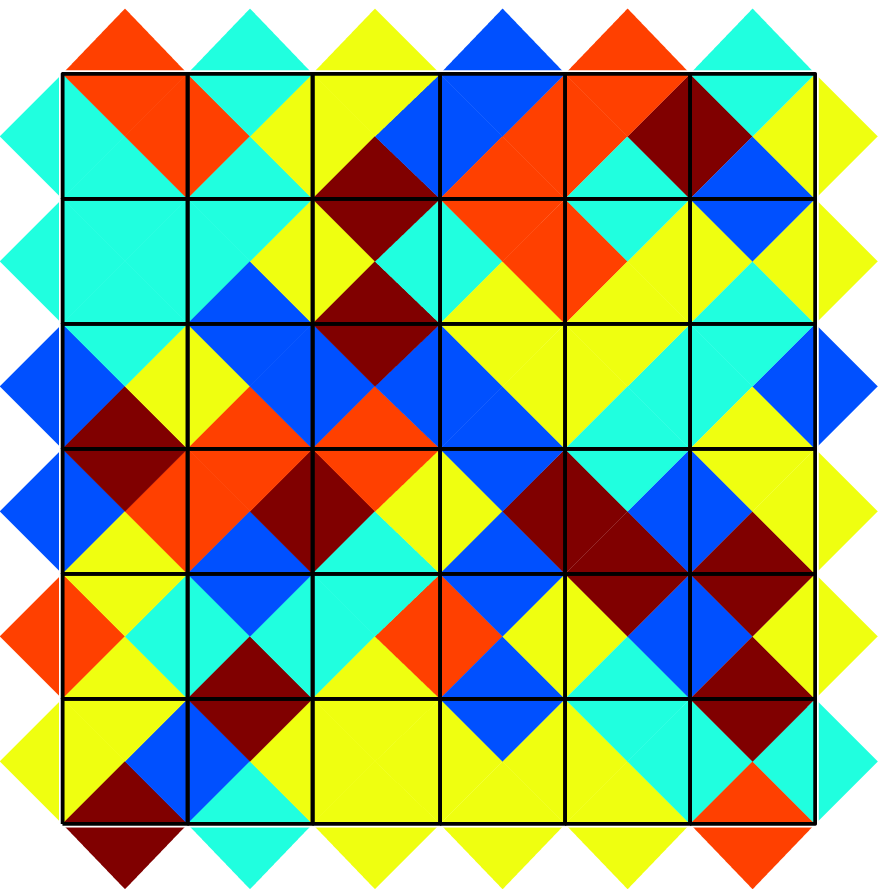}
  \caption{Solved}
\end{subfigure}
\caption{An edge-matching puzzle with 36 square pieces and a frame. Square puzzle pieces are overlaid with four colored triangles, representing the colors of their four edges. The puzzle frame is similarly represented by colored triangles placed along its boundaries. Given a scrambled puzzle as shown on the left, the goal is to rearrange the pieces so that the colors match along adjacent edges. A solution generated using our method is shown on the right.}\label{fig:example_6x6}
\end{figure}

In this paper we propose a novel representation for apictorial edge-matching puzzle games in terms of algebraic varieties, \ie, as solutions of systems of polynomial equations derived from the pieces of the puzzle. We explain how to generate systems of polynomial equations which are satisfied by puzzle solutions. We refer to systems for which the converse also holds, that is, any solution of the system is a solution of the puzzle, as \emph{complete} representations. We characterize and prove the completeness of representations for 2-dimensional translation only puzzles (\ie puzzles with 2-dimensional pieces which can be translated but not rotated).

Using our algebraic representation we devise new algorithms for solving edge-matching puzzles.
We show that approximate solutions can be generated by solving a sequence of \emph{continuous} and \emph{global} convex optimization problems. Our motivation for seeking a global solution strategy is the unique characteristic of edge-matching puzzle problems: only an entire (global) solution provides a certificate of correctness for any local part of the solution. Convergence to a global solution of the original non-convex problem, which corresponds to a solution of the puzzle, is achieved for some interesting puzzle instances.

Specifically, we propose two types of convex relaxations for 2-dimensional translation only edge-matching puzzles. For puzzles where we know in advance that pieces may only be placed in a finite set of predetermined locations, we propose a \emph{Vandermonde formulation} where a solution to the puzzle is attained by approximately solving a constrained optimization problem over the manifold of Vandermonde matrices using linear programming.
For the more general setting, we propose a \emph{rank one formulation} where we approximate a solution to a constrained optimization over rank one matrices using semidefinite programming.

Finally, we show how our computational framework and algorithms can be extended to various interesting variants of edge-matching puzzles including higher-dimensional puzzles, puzzles in which the pieces have irregular shapes (\eg, see Figure~\ref{fig:example_tangram_B}) and puzzles for which not only the location but also the orientation of the pieces is unknown (\eg, see Figure~\ref{fig:example_rotations}).

Effective representations and algorithms for solving puzzle games have applications beyond mere theoretical interest. Computational solutions of puzzles are used in computer aided reconstruction in archaeology \cite{Koller2006,Brown2008,CristinaDaGamaLeito1998,Leitao2005}, the recovery of shredded documents \cite{Justino2006,Marques2009} or photos \cite{Liu2010}, and in image editing \cite{Butman2008,Cho2010a}. They are also relevant to speech descrambling \cite{Zhao2007}, machine translation of text \cite{Levison1967} and even biology (DNA sequence reconstruction can be viewed as a jigsaw puzzle) \cite{Marande2007} and chemistry (determination of molecular conformation) \cite{Lozano-Perez2000}.

\section{Related Work}
Early work, beginning with Freeman and Gardner \cite{Freeman1964} from 1964, develops algorithmic solutions for jigsaw puzzle games based solely on their geometry, a detailed review can be found in \cite{Tybon2004}. Subsequent algorithms combine shape with the image content of the pieces \cite{Kosiba1994,Chung1998,Yao2003,Makridis2006,Nielsen2008}.

A number of methods for solving pictorial puzzles with square pieces have been proposed. Cho et al. \cite{Cho2010} present a probabilistic framework based on the patch transform proposed in \cite{Butman2008,Cho2010a}, which synthesizes an image from a set of image patches. Approximate puzzle reconstruction is achieved via loopy belief propagation on a suitable graphical model. More recently, Pomerantz et al. \cite{Pomeranz2011} and Gallagher \cite{Gallagher2012} proposed greedy algorithms that solve puzzles comprising thousands of pieces. While \cite{Pomeranz2011} solves only for translations in the plane, \cite{Gallagher2012} and \cite{kilho_eccv2014} allow pieces to rotate and solve for their orientations as well. Methods based on non-convex constrained quadratic programming \cite{Andalo2012} as well as genetic algorithms \cite{Sholomon2013} have shown competitive results. All of these approaches however, rely on the statistics of natural images (either explicitly or implicitly). Therefore, they are not expected to perform well if directly applied to apictorial edge-matching puzzles, in which there is no image content.

Specific puzzles have also been studied: Deutsch and Hayes \cite{Deutsch1972} suggest a heuristic approach for solving the Tangram puzzle. A connectionist approach to solving the same puzzle was proposed by Oflazer \cite{Oflazer1993}. Dattorro \cite{Dattorro2005} proposed an interesting convex semidefinite programming approach for attempting to solve the Eternity II puzzle.

Demaine and Hearn \cite{Demaine2001} and Demaine and Demaine \cite{Demaine2007} studied jigsaw puzzles from the perspective of combinatorial game theory. In the latter, the decision problems corresponding to jigsaw puzzles and edge-matching puzzles (\ie, ``does this puzzle have a solution") are shown to be NP-complete.
\section{Problem Statement}
\label{section:probem_statement}
\begin{figure}[t]
  \centering
  \includegraphics[width=0.9\textwidth]{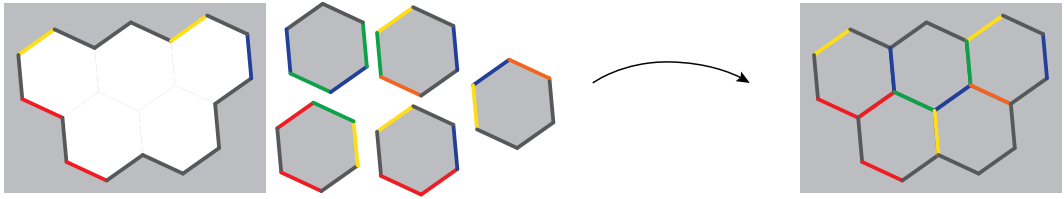}
  \caption{Problem statement --  place all pieces within a given bounding frame such that all edges pair with matching edges.}
  \label{fig:problem_statement}
\end{figure}
In this work, we focus on \emph{ideal edge-matching framed} puzzles. That is, apictorial geometric puzzles, in which a set of pieces of known shapes and edge colors is given. The puzzle is bounded by a frame and each edge must perfectly match (in its location, color and orientation) an edge of either another piece or the frame.

We start by considering simple polygonal 2-dimensional puzzles, where puzzle pieces are equilateral polygons with unit-length colored edges. Furthermore, we will assume that puzzle pieces may only be translated (\ie, the rules of the game allow shifting the pieces but not rotating them).

The goal of the game is then to place all pieces within the given bounding frame, such that all edges pair with matching edges. This, in turn, implies that the puzzle pieces cover its frame (for a properly designed puzzle), see Figure~\ref{fig:problem_statement}.

Next, we formally present an algebraic representation for these types of puzzles. In Section~\ref{sect:extensions} we show how to extend our framework to cases in which the above assumptions do not hold such as puzzles with irregularly shaped pieces, higher-dimensional puzzles and puzzles in which rotating the pieces is also allowed.

\subsection{Formulation}
Consider a 2-dimensional polygonal puzzle with $N$ pieces. As illustrated in Figure~\ref{fig:two_matching_edges}, we describe the $i$'th piece by the location $t_i \in \Real^2$ of its center and its set of edge elements $E_i$ (polygon sides). Each edge $j\in E_i$ is described by the relative location of its center $b_{i,j}$ with respect to the piece center, its color $c_{i,j}$ and its orientation $\theta_{i,j}$. Therefore, the absolute location of the $j$'th edge of the $i$'th piece is given by the sum $t_i+b_{i,j}$.

We shall consider puzzles with a bounding frame, which can be seen as another puzzle piece, with the exception of being static. We denote the piece corresponding to the puzzle frame by $i=0$ and the properties of its edge elements by $b_{0,j}$, $c_{0,j}$ and $\theta_{0,j}$.

Under the assumption that puzzle pieces may \emph{only be translated}, the goal of the game amounts to finding $t_1,\ldots,t_N$ for which all edge elements pair with matching edge elements in their spatial locations, colors and orientations (see Figure~\ref{fig:problem_statement}). We shall refer to a configuration $t_1,\ldots,t_N$ that satisfies this criterion as a \emph{solution} of the puzzle.

\begin{figure}[h!]
  \centering
  \includegraphics[width=0.7\textwidth]{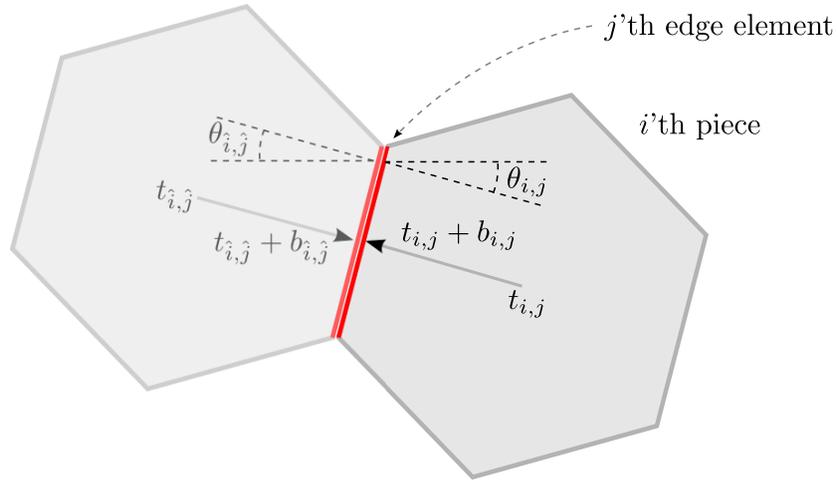}
  \caption{The description of edge elements and the constraints implied by two matching edge elements}
  \label{fig:two_matching_edges}
\end{figure}

\section{Algebraic Representation of a Puzzle}
\label{section:alg_representation}
In this section, we derive an algebraic representation of a puzzle. We show how, given a puzzle, one can construct an algebraic system of equations that characterizes the puzzle. Furthermore, we prove that this representation is complete, in the sense that it encapsulates all the information available in the given puzzle and its solutions are exactly the solutions of the puzzle.

Suppose $t_1,\ldots,t_N$ is a solution of the puzzle. By definition, every edge is paired with a matching edge. Consider such a pair of matching edges,  edge $j$ of the $i$'th piece and  edge $\hat{j}$ of the $\hat{i}$'th piece. By definition, these edges must spatially coincide, \ie,
\begin{equation}
  t_i+b_{i,j} = t_{\hat{i}}+b_{\hat{i},\hat{j}}. \label{eqn:piece_match_pos}
\end{equation}
They must have the same color
\begin{equation}
  c_{i,j} = c_{\hat{i},\hat{j}}, \label{eqn:piece_match_color}
\end{equation}
and they must face opposite directions
\begin{equation}
  \theta_{i,j} \equiv \theta_{\hat{i},\hat{j}}+\pi \quad (\Mod \ 2\pi). \label{eqn:piece_match_theta}
\end{equation}
See Figure~\ref{fig:two_matching_edges} for an illustration. In particular, notice that~\eqref{eqn:piece_match_pos} linearly constrains the relative locations of the two pieces. For simplicity of notation we will omit the congruent modulo notation for angles with the understanding that $\theta = \phi$ is shorthand for $\theta \equiv \phi \ (\Mod \ 2\pi)$.

In general, however, we do not know in advance which edges pair with each other. Nevertheless, we notice that all edges of a certain color and orientation $(c,\theta)$ must pair with complementary edges of the same color and opposite orientation $(c,\theta+\pi)$. Therefore, for every $(c,\theta)$ we have an equality of sets
\begin{equation}
\set{t_i+b_{i,j} : \begin{array}{l}
c_{i,j} = c \\
\theta_{i,j}=\theta \\
\end{array}} =
  \set{t_i+b_{i,j} : \begin{array}{l}
c_{i,j} = c \\
\theta_{i,j}=\theta+\pi \\
\end{array}}.
\label{eqn:eq_edges_set}
\end{equation}
Clearly, equation~\eqref{eqn:eq_edges_set} holds for all $(c,\theta)$ if and only if $t_1,\ldots,t_N$ is a solution of the puzzle.

If we define the signed indicator function w.r.t. $(c,\theta)$ by
\begin{equation}
s_{i,j}(c,\theta) = \left\{
\begin{array}{cl}
1, & c_{i,j}=c \text{ and } \theta_{i,j}=\theta \\
-1,& c_{i,j}=c \text{ and } \theta_{i,j}=\theta+\pi \\
0 & \text{otherwise} \\
\end{array}
 \right.
\end{equation}
then, by construction we have derived the following result:

\bigskip
\begin{proposition} If $t_1,\ldots,t_N$ is a solution of the puzzle then
\begin{equation}
\sum_{i,j} s_{i,j}(c,\theta) f\parr{t_i+b_{i,j}} = 0
\label{eqn:eq_basic_condition}
\end{equation}
for every $(c,\theta)$ and every real valued function $f:\Real^2\to\Real$.
\label{proposition:basic_condition}
\end{proposition}
\bigskip

Note that~\eqref{eqn:eq_basic_condition} is an algebraic constraint on $t_1,\ldots,t_N$. This allows us to construct an algebraic \emph{system} of constraints by taking all admissible  $(c,\theta)$'s and different $f$'s.
Different choices of a function $f$ establish systems with possibly different properties.

\subsection{Linear Representation}
A particular choice of $f$ is $f_{v}(u)=v^Tu$ for any $v\in\Real^2$. This choice yields a linear representation of the puzzle. Namely, equations of the form
\begin{equation}
\sum_{i,j} s_{i,j}(c,\theta)\parr{t_i+b_{i,j}} = 0
\label{eqn:eq_basic_linear}
\end{equation}
for all $(c,\theta)$ define a \emph{linear system of equations} in $t_1,\ldots,t_N$. Clearly, every solution of the puzzle must satisfy this system of equations. Usually, however, the converse does not hold, and this linear system does not determine the solutions of the puzzle.

\subsection{Polynomial Representation}
\label{section:poly_representation}
In the rest of the paper, we use functions $f$ which belong to a family of \emph{exponential} functions. For a given $k\in\Real^2$, we define $f_k(u)=\Exp{k^Tu}$. Equation~\eqref{eqn:eq_basic_condition} then becomes
\begin{equation}
\sum_{i,j} s_{i,j}(c,\theta) \Exp{k^T\parr{t_i+b_{i,j}}} = 0.
\label{eqn:eq_basic_exp1}
\end{equation}
To simplify notations we change variables $T_i = \Exp{t_i}$ (element-wise exponentiation) to obtain
\begin{equation}
\sum_{i} \alpha_i^{(k,c,\theta)}T_i^k = 0,
\label{eqn:eq_basic_exp2}
\end{equation}
where $\alpha_i^{(k,c,\theta)}=\sum_j {s_{i,j}(c,\theta) \Exp{k^T b_{i,j}}}$, and $T_i^k$ is defined according to the multi-index notation
$$
T_i^k=(e^{t_i})^k=e^{k^T t_i}.
$$

Note that~\eqref{eqn:eq_basic_exp2} is simply a polynomial in (the entries of) $T_1,\ldots,T_N$. Collecting many equations of the form~\eqref{eqn:eq_basic_exp2} for all $(c,\theta)$ and various choices of $k$ establishes a \emph{polynomial system of equations} in $T_1,\ldots,T_N$.

According to Proposition~\ref{proposition:basic_condition}, if $t_1,\ldots,t_N$ is a solution of the puzzle then $T_1,\ldots,T_N$ is a solution of any such polynomial system of equations. Next, we discuss the conditions under which the converse holds, so that every solution of the polynomial system corresponds to a valid solution of the puzzle. In what follows, we identify $t_1,\ldots,t_N$ with $T_1,\ldots,T_N$, and thus refer to the latter as a solution of the puzzle as well.

\subsection{Completeness of Polynomial Representations}
Naturally, we seek a complete representation, one that encapsulates all the information available in the given puzzle and exactly characterizes the solutions of the puzzle.

We consider the polynomial system constructed by collecting equations of the form~\eqref{eqn:eq_basic_exp2} for different $(c,\theta)$ and various $k$ values.
Intuitively, if additional equations provide additional independent constraints then we can hope that collecting many such equations will sufficiently constrain the unknown variables, exactly determining the solutions of the puzzle. Indeed, we can show that a system of sufficiently many such polynomial equations establishes an equivalent representation, in which $T_1,\ldots,T_N$ is a solution of the puzzle if and only if it is a root of this system.

More formally, consider the polynomial system $\PP$ constructed by collecting equations as follows: for every possible type of edge $(c,\theta)$ add $K(c,\theta) \in \Natural$  equations of the form~\eqref{eqn:eq_basic_exp2}:
\begin{equation*}
\sum_{i} \alpha_i^{(k,c,\theta)}T_i^k = 0
\end{equation*}
for distinct values of $k$, where $K(c,\theta)$ is an integer which depends on the edge type $(c,\theta)$. This construction is well defined since, by our assumption of polygonal puzzle pieces, there exist finitely many distinct edge types $(c,\theta)$.

By construction, every solution of the puzzle is a solution of $\PP$.
The proposition below suggests that the converse also holds, provided that $\PP$ is defined by sufficiently many equations. In particular, Appendix~\ref{section:proof_equivalence} provides a \emph{constructive} proof which shows that the number of equations required for each type of edge $(c,\theta)$ is exactly the number of puzzle edges of this type.

 \bigskip
\begin{proposition}
\label{proposition:equivalent_representation}
There exists (constructively) a polynomial system $\PP$ with
\begin{equation*}
K(c,\theta) = \#\set{\text{edges of type } (c,\theta)}
\end{equation*}
equations for each edge type $(c,\theta)$ that satisfies that $T_1,\ldots,T_N$ is a solution of $\PP$ if and only if it is a solution of the puzzle.
\end{proposition}

\bigskip
The proof relies on the following observation. Consider \emph{independently} all $K=K(c,\theta)$ edges of type $(c,\theta)$. Any pairing of these edges with corresponding opposite edges satisfies polynomial equations of the form~\eqref{eqn:eq_basic_exp2} for any order $k$. There are exactly $K!$ such pairings, as the number of permutations of $K$ elements. On the other hand, using B\'ezout's theorem \cite{Coolidge2004, Sottile2002} we can show that a polynomial system of equations of the form~\eqref{eqn:eq_basic_exp2} of orders $k=1,\ldots,K$ has no more than $K!$ distinct solutions. Thus, such a system exactly characterizes all possible independent pairings of $(c,\theta)$ edges,  in the sense that any subset of its equations is insufficient whereas additional higher-order polynomials are redundant. The full proof (given in Appendix~\ref{section:proof_equivalence}) follows by generalizing this observation to the intersection of multiple similar systems, each corresponding to edges of a different type $(c,\theta)$, and the addition of constraints accounting for the geometry of the puzzle pieces.

\section{Solving Puzzles}
\label{section:solving}
In previous sections we have shown that a puzzle problem can be faithfully represented as a system of polynomial equations
\begin{equation}
\set{\sum_{i} \alpha_i^{(k,c,\theta)}T_i^k = 0}_{k,c,\theta}
\label{eqn:eq_poly_system}
\end{equation}
in the unknowns $T_1,\ldots,T_N$.

Next, we discuss several approaches for determining a solution $T_1,\ldots,T_N$ of this system, either directly or by approximation. Once such a solution is found, one may recover the  piece locations $t_1,\ldots,t_N$ using the relation $T_i=\Exp{t_i}$. Assuming the conditions for Proposition~\ref{proposition:equivalent_representation} hold, these locations are guaranteed to solve the puzzle.

To simplify notations we present the derivations for 1-dimensional puzzles (with coordinates over $\Real$). Extending the results to the case of 2-dimensional puzzles, discussed so far,  (or in fact arbitrary dimension) is straightforward, see Appendix~\ref{section:solution_2d_techincal} for the technical details.

\subsection{Exact Solutions}
Solving polynomial systems of equations is an important problem which is covered in a large body of literature, for example see \cite{Cohen1999,Sturmfels2002,Cox2007,Laurent2012} and the references therein. Existing methods for solving generic systems may be coarsely classified into two categories: symbolic (exact) and numerical methods.

Gr\"obner algorithms \cite{Buchberger1976} are typical representatives of the class of symbolic methods. These algorithms seek to simplify a given polynomial system, to enable the extraction of its roots. This can be seen as a multivariate, non-linear generalization of the Euclidean algorithm (for computation of univariate GCD) and Gaussian elimination for linear systems. However, the complexity of methods for calculating Gr\"obner bases may be extremely high.

A popular approach in the class of numerical methods is homotopy continuation. Such methods rely on Bertini's theorem which introduces a continuous deformation between the polynomial system whose solution we seek and a simpler polynomial system whose solutions are known, see \cite{Li1997,Sommese2005} for more details. Keeping track of the roots during this deformation can provide a solution to the system. However, a good initialization (\ie, an initial polynomial system with known solutions) is crucial to the success of the algorithm. Unfortunately, the combinatorial nature of the polynomial systems we are concerned with typically makes it hard to find a good initialization.

\subsection{Reformulation and Convex Relaxation}
\label{section:reformulation}
Our attempts to solve very simple puzzles using generic polynomial system solvers met with limited success. Therefore, in what follows we describe a reformulation and optimization algorithms which try to take advantage of the unique properties of systems of the form~\eqref{eqn:eq_poly_system}.

Note that the polynomial system~\eqref{eqn:eq_poly_system} has a special structure: each one of its equations includes only $k$'th degree monomials in the unknown $T_i$'s. We exploit this structure in order to restate the problem of solving puzzles as a linear system of equations over a (non-linear) manifold of matrices. We discuss two alternative reformulations, relying on either Vandermonde or rank one matrices. We further discuss convex relaxations of the resulting problems, which can be used to solve certain types of puzzles.

\subsubsection{Vandermonde Reformulation and Preset Locations}
\label{section:reformulation_vandermonde}
Any polynomial system of equations can always be replaced with a linear system of equations of a higher dimension, by considering each monomial as a variable, and adding simple polynomial constraints coupling the new variables. In the case of our polynomial systems, this can be conveniently stated using Vandermonde matrices. Recall that
a matrix of the form
\begin{equation}
\begin{bmatrix}
1 & T_1 & T_1^2 & \cdots & T_1^K \\
1 & T_2 & T_2^2 & \cdots & T_2^K \\
1 & T_3 & T_3^2 & \cdots & T_3^K \\
\vdots & \vdots & \vdots & \ddots & \vdots \\
1 & T_N & T_N^2 & \cdots & T_N^K \\
\end{bmatrix}
\label{eqn:vandermonde_matrix}
\end{equation}
is known as a \emph{Vandermonde} matrix. It is straightforward to see that the problem of solving~\eqref{eqn:eq_poly_system} can be equivalently reformulated as
\begin{equation}
\begin{array}{ll}
\find & X = \brac{x_{ij}} \\
\st & X \text{\ is Vandermonde of size $N \times K$}  \\
& \sum_{i} \alpha_i^{(k,c,\theta)}x_{ij} = 0 \qquad \forall j,c,\theta. \\
\end{array}
\label{eqn:eq_equiv_problem_vandermonde}
\end{equation}
To ensure that $X$ contains the monomials necessary for a complete polynomial representation as in Proposition~\ref{proposition:equivalent_representation} we choose $K \geq \max_{(c,\theta)}K(c,\theta)$.
Formulation~\ref{eqn:eq_equiv_problem_vandermonde} amounts to solving a linear system over the manifold of Vandermonde matrices. Of course this new problem is as hard to solve as the original problem itself. We will now show how to derive an approximation under an additional assumption.

\paragraph{Preset Locations}
In many puzzles, pieces may only be placed at a finite number of preset locations. This is the case, for example, in square tiling puzzles (\eg, see Figure~\ref{fig:example_6x6}). This additional information can be incorporated into the formulation~\eqref{eqn:eq_equiv_problem_vandermonde}.

Suppose that piece locations $t_1,\ldots,t_N$ must belong to a finite \emph{known} set of feasible locations $\{s_1,\ldots,s_N\}$. This assumption can be formally stated by

\begin{equation}
\begin{bmatrix}t_1 \\ \vdots \\ t_N\end{bmatrix}
 = P
 \begin{bmatrix}s_1 \\ \vdots \\ s_N\end{bmatrix}
\end{equation}
for some permutation matrix $P$. In turn, since $T_i=\Exp{t_i}$, this implies that
\begin{equation} \label{eqn:T_eq_S_permutation}
\begin{bmatrix}T_1 \\ \vdots \\ T_N\end{bmatrix}
 = P
 \begin{bmatrix}S_1 \\ \vdots \\ S_N\end{bmatrix},
\end{equation}
where $S_i=\Exp{s_i}$. We use the known preset piece locations to define the Vandermonde matrix $Y$ by
\begin{equation}
Y =
\begin{bmatrix}
1 & S_1 & S_1^2 & \cdots & S_1^K \\
1 & S_2 & S_2^2 & \cdots & S_2^K \\
1 & S_3 & S_3^2 & \cdots & S_3^K \\
\vdots & \vdots & \vdots & \ddots & \vdots \\
1 & S_N & S_N^2 & \cdots & S_N^K \\
\end{bmatrix}.
\label{eqn:preset_locations_Y}
\end{equation}
Then, \eqref{eqn:T_eq_S_permutation} implies that any solution of problem \eqref{eqn:eq_equiv_problem_vandermonde} must satisfy
\begin{equation}
X=PY.
\label{eqn:X_eq_Y_permutation}
\end{equation}

Notice that $X=PY$ defines a Vandermonde matrix for any choice of \emph{permutation} matrix $P$. Therefore, problem \eqref{eqn:eq_equiv_problem_vandermonde} can be replaced with the following equivalent problem
\begin{equation}
\begin{array}{ll}
\find & P = [p_{i,j}] \\
\st & P \text{ is a permutation} \\
& X = \brac{x_{ij}} = PY \\
& \sum_{i} \alpha_i^{(k,c,\theta)}x_{i,j} = 0 \qquad \forall j,c,\theta. \\
\end{array}
\label{eqn:eq_equiv_problem_permutation}
\end{equation}
Thus, we have replaced the Vandermonde constraint with the restriction of the search space to the set of permutation matrices, which in turn corresponds to the set of feasible piece locations.

\paragraph{Convex Relaxation} Problem \eqref{eqn:eq_equiv_problem_permutation}
is a non-convex feasibility problem. Nevertheless, we next show that it can be restated as a linearly constrained quadratic maximization problem, and then \emph{relaxed} into a sequence of linear programs.

The Birkhoff-von Neumann theorem \cite{Birkhoff1947} asserts that permutations are the extremal points of the set of bi-stochastic matrices. Moreover, the Frobenius norm of matrices in this set is maximized by permutations. Using this, \eqref{eqn:eq_equiv_problem_permutation} can be replaced with the following equivalent optimization problem:
\begin{equation}
\begin{array}{ll}
\maximize & \norm{P}_F^{2} \\
\st & P\1=\1,\ \1^TP=\1^T,\ P\geq0 \\
& X = \brac{x_{ij}} = PY \\
& \sum_{i} \alpha_i^{(k,c,\theta)}x_{i,j} = 0 \qquad \forall j,c,\theta \\
\end{array}
\label{eqn:eq_reformulate_problem_permutation}
\end{equation}
where $\1 \in \Real^N$ is the all one vector. This is again a non-convex optimization, as it maximizes norm.

We generate approximate solutions to \eqref{eqn:eq_reformulate_problem_permutation} by linearizing its quadratic objective term and applying an iterative algorithm, in similar spirit to reweighted optimization approaches (\eg, reweighted $\ell_1$ minimization \cite{Candes:2008}).

Note that $\norm{P}_F^{2}=\tr\parr{P^TP}=\ip{P,P}$, where $\ip{A,B}=\tr\parr{B^TA}$ is the standard matrix inner product. Suppose that $\hat{P}$ is an approximate guess of $P$. We linearly approximate the quadratic objective as $\langle\hat{P},P\rangle$ and iterate solving the following \emph{linear program} (LP):
\begin{equation}
\begin{array}{ll}
\maximize & \ip{P^{(n-1)},P} \\
\st & P\1=\1,\ \1^TP=\1^T,\ P\geq0 \\
& X = \brac{x_{ij}} = PY \\
& \sum_{i} \alpha_i^{(k,c,\theta)}x_{i,j} = 0 \qquad \forall j,c,\theta \\
\end{array}
\label{eqn:eq_relax_problem_permutation}
\end{equation}
where $P^{(n-1)}$ is the optimizer of the previous iteration. The algorithm is outlined in Algorithm~\ref{alg:solve_approx_lp}.

Note that a global optimizer of \eqref{eqn:eq_reformulate_problem_permutation} is a fixed point of \eqref{eqn:eq_relax_problem_permutation}. We  initialize \eqref{eqn:eq_relax_problem_permutation} with $P^{(0)}=0$. With this initialization \eqref{eqn:eq_relax_problem_permutation} is equivalent to the standard LP relaxation of the feasibility problem \eqref{eqn:eq_equiv_problem_permutation} over the convex hull of permutations. This sequence of optimization problems attains a bounded and monotonically non-decreasing objective. However, convergence to a permutation is not guaranteed.

\begin{algorithm}
\caption{Iterative LP approximation for puzzles with preset piece locations}
\KwIn{
\ \ Polynomial representation coefficients $\alpha_i^{(k,c,\theta)}$ \\
\qquad \qquad \ \ Preset locations matrix $Y$ as in \eqref{eqn:preset_locations_Y}
}
\KwOut{Puzzle solution matrix $X$}
\BlankLine
\BlankLine
Initialize $P^{(0)}=0;\ n = 0$\;
\Repeat{$P^{(n)}$ is a permutation}
{
$n=n+1$\;
Solve the linear program \eqref{eqn:eq_relax_problem_permutation} with $P^{(n-1)}$\;
Set $P^{(n)}$ to be the optimizer\;
}
\Return{$X=P^{(n)}Y$}\;
\label{alg:solve_approx_lp}
\end{algorithm}

\paragraph{Examples} We demonstrate the proposed approach by solving examples of edge-matching puzzles. Specifically, we generate 2-dimensional square tiling puzzles of various sizes whose edge colors are drawn at random,  see Figures~\ref{fig:example_6x6} and~Figure~\ref{fig:example_8x8}.

We use the polynomial representation described in Section~\ref{section:poly_representation} to calculate the coefficients of a corresponding polynomial system of
equations and apply Algorithm~\ref{alg:solve_approx_lp}. Figure~\ref{fig:example_6x6} shows the method applied to a $6\times6$ puzzle. Only 2 iterations of \eqref{eqn:eq_relax_problem_permutation} were required to attain a permutation, that corresponds to the sole solution of the puzzle.

Figure~\ref{fig:example_8x8} depicts the method applied to an $8\times8$ puzzle. 6 iterations of \eqref{eqn:eq_relax_problem_permutation}, shown at the bottom of the figure, were required to attain a permutation. In this case, the algorithm collapses into one of the two solutions of this puzzle.
\begin{figure}[t]
\centering
\begin{subfigure}[b]{0.49\textwidth}
  \centering
  \includegraphics[width=0.7\columnwidth]{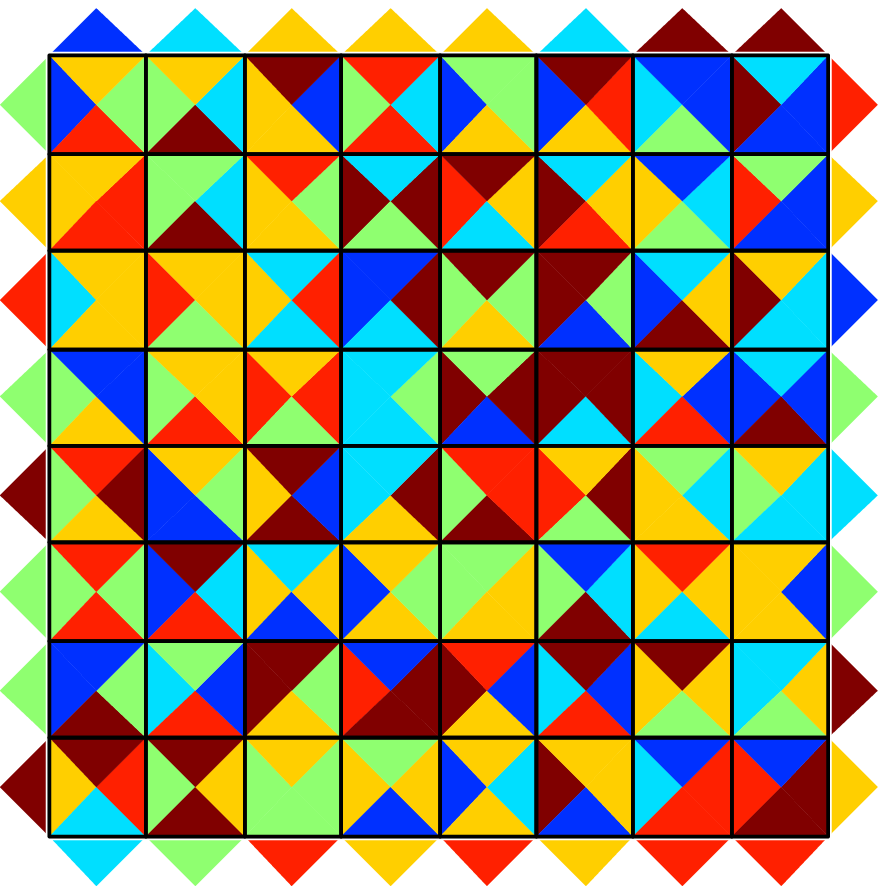}
  \caption{Scrambled}
\end{subfigure}
\begin{subfigure}[b]{0.49\columnwidth}
  \centering
  \includegraphics[width=0.7\columnwidth]{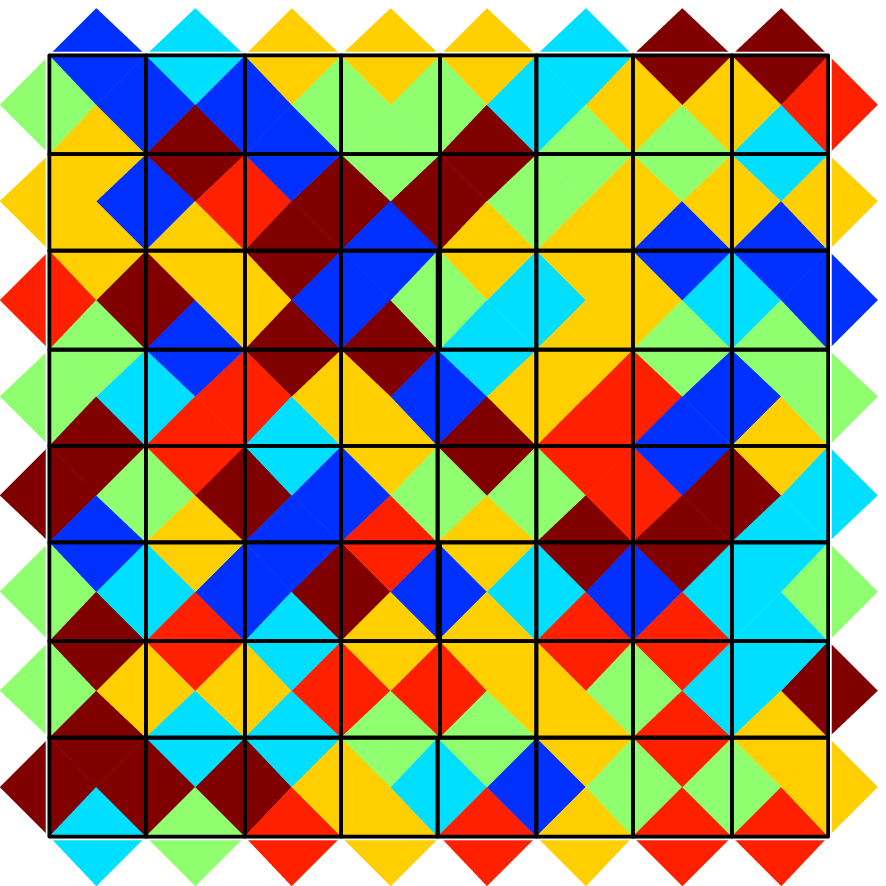}
  \caption{Solved}
\end{subfigure}
\\
\begin{subfigure}[b]{1\columnwidth}
  \centering
  \includegraphics[width=1\columnwidth]{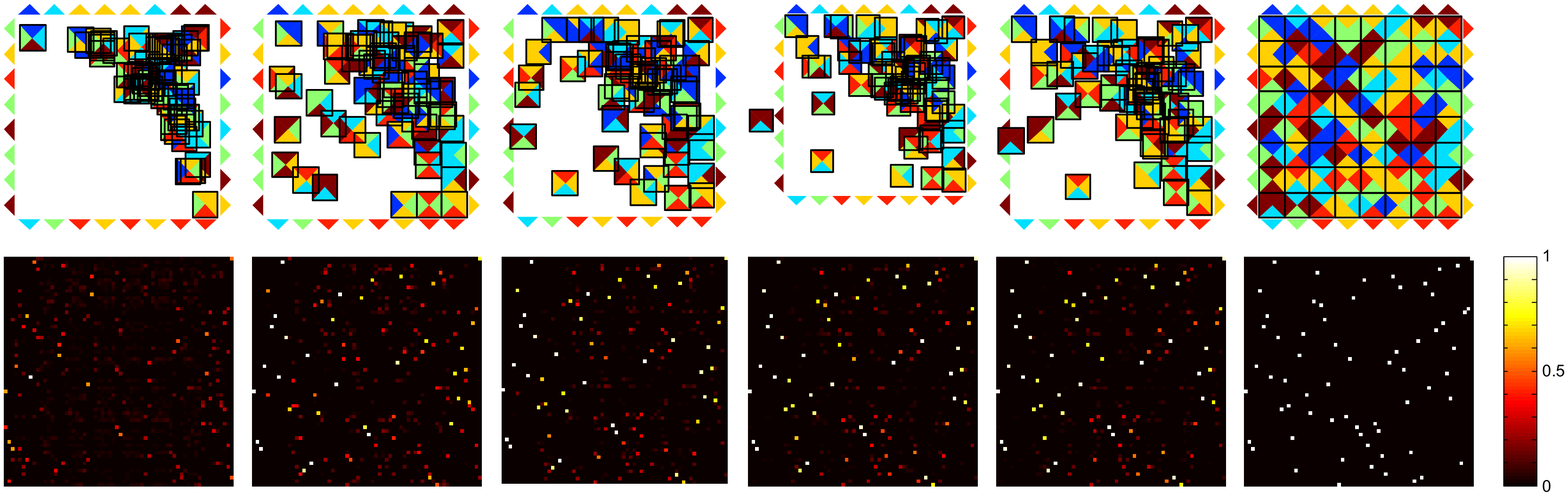}
  \caption{Iterations of \eqref{eqn:eq_relax_problem_permutation}, left to right. Top row shows intermediate puzzle solutions. Bottom row shows the state of the variable $P$ (a $64\times 64$ matrix), converging to a permutation matrix. White entries indicate a value of 1, black is 0 and intermediate colors correspond to fractional values.}
\end{subfigure}
\caption{Solution of an $8 \times 8$ puzzle using our method.}\label{fig:example_8x8}
\end{figure}

\subsubsection{Rank One Reformulation}
\label{section:reformulation_rank_one}
For certain types of puzzles the set of valid locations of the pieces in a solution is not known in advance. One such example, where due to the irregularity of piece shapes their possible locations cannot be predetermined, is the \emph{Tangram} puzzle (Figures~\ref{fig:example_tangram_A} and \ref{fig:example_tangram_B}) In this section we propose an alternative reformulation of~\eqref{eqn:eq_poly_system} as an optimization over rank one matrices.

Optimization subject to rank one constraints, or more generally low rank optimization, has recently received substantial interest (\eg, \cite{Orsi2006,Kulis2007,Mitra2010,Recht2010}). We suggest a convex relaxation algorithm for generating approximate solutions to our rank one formulation, which does not require the piece locations to be known in advance.

Define $N$ \emph{Hankel} matrices $Z_1,\ldots,Z_N$ of the form
\begin{equation}
Z_i =
\begin{bmatrix}1 & z_{i,1} & z_{i,2} & \cdots & z_{i,K} \\
z_{i,1} & z_{i,2} &  & \iddots & z_{i,K+1} \\
z_{i,2} &  & \iddots &  & \vdots \\
\vdots & \iddots &  & \iddots &  \\
z_{i,K} & z_{i,K+1} & \cdots &  & z_{i,2K} \\
\end{bmatrix}.
\label{eqn:eq_hankel_matrix}
\end{equation}
We state the following problem
\begin{equation}
\begin{array}{ll}
\find & Z_1,\ldots,Z_N \text{ of the form \eqref{eqn:eq_hankel_matrix}} \\
\st & \sum_{i} \alpha_i^{(k,c,\theta)}z_{i,j} = 0 \qquad \forall j,c,\theta\\
& z_{i,j} \geq 0 \qquad \forall i,j \\
& \rank(Z_i)=1 \qquad \forall i. \\
\end{array}
\label{eqn:eq_equiv_problem_hankel}
\end{equation}
It is clear that if $T_1,\ldots,T_N$ is a solution of the puzzle then the assignment
\begin{equation}
Z_i =
\begin{bmatrix}1 & T_i & T_i^2 & \cdots & T_i^K \\
T_i & T_i^2 &  & \iddots & T_i^{K+1} \\
T_i^2 &  & \iddots &  & \vdots \\
\vdots & \iddots &  & \iddots &  \\
T_i^{K} & T_i^{K+1} & \cdots &  & T_i^{2K} \\
\end{bmatrix}
=
\begin{bmatrix}
1 \\ T_i \\ T_i^2 \\ \vdots \\ T_i^K \\
\end{bmatrix}
\begin{bmatrix}
1 \\ T_i \\ T_i^2 \\ \vdots \\ T_i^K \\
\end{bmatrix}^T
\label{eqn:eq_hankel_matrix_Ts}
\end{equation}
is a solution of problem \eqref{eqn:eq_equiv_problem_hankel}. To show that \eqref{eqn:eq_equiv_problem_hankel} is in fact equivalent to the problem of solving the polynomial system \eqref{eqn:eq_poly_system} we must show that any solution to \eqref{eqn:eq_equiv_problem_hankel} is also a solution of the polynomial system \eqref{eqn:eq_poly_system}.

Towards this end, we use the following simple lemma:
\bigskip
\begin{lemma}
\ $\rank(Z_i)=1$ and $z_{i,j}\geq0$ if and only if $z_{i,j} = z_{i,1}^j$ for all $j$.
\end{lemma}
\bigskip

The lemma follows immediately by noticing that $Z_i$ must be an outer product of its first column with itself. This implies that, under the constraints of problem \eqref{eqn:eq_equiv_problem_hankel}, the matrix $Z_i$ is simply the outer product of the $i$'th row of the Vandermonde matrix \eqref{eqn:vandermonde_matrix} with itself. Thus, every solution of the low-rank problem \eqref{eqn:eq_equiv_problem_hankel} is a solution of the Vandermonde problem \eqref{eqn:eq_equiv_problem_vandermonde}, and in turn of the original polynomial system \eqref{eqn:eq_poly_system}.

Therefore, \eqref{eqn:eq_equiv_problem_hankel} is an alternative formulation to the problem of solving puzzles. It can be interpreted as finding a solution of a linear system over non-linear manifolds of rank one matrices.

\paragraph{Convex Relaxation} Problem \eqref{eqn:eq_equiv_problem_hankel}
is a non-convex feasibility problem. Note that its constraints $\rank(Z_i)=1$ and $z_{i,j}\geq0$ imply that $Z_i$ is a positive semidefinite matrix, which we denote by $Z_i\succeq0$. Namely, since $Z_i$ is a symmetric rank one matrix, it admits a rank one eigendecomposition of the form $Z_i=\lambda_i u_i u_i^T$ for some vector $u_i$. Having non-negative entries implies that $\lambda_i\geq0$, and in turn that $Z_i\succeq0$.

Dattorro \cite{Dattorro2005} discusses a semidefinite programming (SDP) heuristic for rank-constrained optimization. Inspired by his work, we propose an iterative approximate procedure. Intuitively, the idea is that a rank one symmetric positive semidefinite matrix can be characterized as one that minimizes the sum of all eigenvalues but the largest. Thus, a feasibility problem of the form \eqref{eqn:eq_equiv_problem_hankel} can be cast as an eigenvalue minimization problem. Since the resulting problem is not convex we generate approximate solutions by employing local linearization and solving a sequence of optimization problems. We iteratively solve the following SDP:
\begin{equation}
\begin{array}{ll}
\minimize & \ip{W_1,Z_1}+\cdots+\ip{W_N,Z_N} \\
\st
& Z_1,\ldots,Z_N \text{ are of the form \eqref{eqn:eq_hankel_matrix}}\\
& \sum_{i} \alpha_i^{(k,c,\theta)}z_{i,j} = 0 \qquad \forall j,c,\theta\\
& Z_i\succeq0 \qquad \forall i \\
\end{array}
\label{eqn:eq_equiv_problem_hankel_SDP}
\end{equation}
where $W_i$ are fixed matrices, updated at each iteration according to
\begin{equation}
W_i = V_i
\begin{bmatrix}0 &  &   & 0 \\
  & 1 &   &   \\
  &   & \ddots &   \\
0 &   &   & 1 \\
\end{bmatrix}
V_i^T
\label{eqn:eq_equiv_problem_hankel_SDP_Wi_update}
\end{equation}
and $Z_i^{(n-1)}=V_i\Lambda_iV_i^T$ is the eigendecomposition of $Z_i^{(n-1)}$, the optimizer of the previous iteration, with eigenvalues sorted in descending order.

Note that the functional of \eqref{eqn:eq_equiv_problem_hankel_SDP} is non-negative, as  $\set{Z_i}$ and $\set{W_i}$ are all positive semidefinite. Moreover, it vanishes if and only if the rank of each of the matrices $Z_i$ is at most one. Therefore, it is easy to see that a solution of \eqref{eqn:eq_equiv_problem_hankel} is a global optimizer of \eqref{eqn:eq_equiv_problem_hankel_SDP}, as well as a fixed point of the suggested iterative procedure (see \cite{Dattorro2005} for additional details).

We initialize the iterative procedure with $W_i=0$, which reduces \eqref{eqn:eq_equiv_problem_hankel_SDP} to a standard SDP relaxation for the rank one feasibility problem \eqref{eqn:eq_equiv_problem_hankel}. The algorithm is outlined in Algorithm~\ref{alg:solve_approx_sdp}. As in Section~\ref{section:reformulation_vandermonde}, convergence to a global minimum is not guaranteed.

\begin{algorithm}
\caption{Iterative SDP approximation for puzzles}
\KwIn{\ \ Polynomial representation coefficients $\alpha_i^{(k,c,\theta)}$\\
\qquad \qquad \ \ tolerance $\varepsilon$}
\KwOut{Puzzle solution matrices $Z_1,\ldots,Z_N$ }
\BlankLine
\BlankLine
Initialize $W_1,\ldots,W_N=0;\ n = 0$\;
\Repeat{$\ip{W_1,Z_1^{(n)}}+\cdots+\ip{W_N,Z_N^{(n)}}\leq\varepsilon$}
{
$n=n+1$\;
Solve the semidefinite program \eqref{eqn:eq_equiv_problem_hankel_SDP} with $W_1,\ldots,W_N$\;
Set $Z_1^{(n)},\ldots,Z_N^{(n)}$ to be the optimizer\;
Calculate the eigendecompositions $Z_i^{(n)}=V_i\Lambda_iV_i^T$\;
Update $\{W_i\}$ according to \eqref{eqn:eq_equiv_problem_hankel_SDP_Wi_update}\;
}
\Return{$Z_1=Z_1^{(n)},\ldots,Z_N=Z_N^{(n)}$}\;
\label{alg:solve_approx_sdp}
\end{algorithm}

\paragraph{Examples}
To demonstrate the proposed approach, we applied the iterative algorithm to the problem of solving \emph{Tangram} puzzles. Figures~\ref{fig:example_tangram_A} and~\ref{fig:example_tangram_B} show two instances of tangram puzzles. In these puzzles one aims to tile a given shape (b) with the puzzle pieces (a). Thus, it can be seen as an edge-matching problem in which all edges share the same color. In the figures, we color each piece in a different color for visualization only.

The coefficients of a corresponding polynomial system of equations were calculated as described in Section~\ref{section:poly_representation}. (The calculations were slightly adapted to allow for non-equilateral polygonal pieces, see Section~\ref{sect:extensions_piece_shape} for additional technical details.) Then, the solution (c) was obtained by applying Algorithm~\ref{alg:solve_approx_sdp}.
(d) shows the iterations of the algorithm until convergence. Note that no assumptions on the locations of the puzzle pieces were made, in contrast to the algorithm derived in Section~\ref{section:reformulation_vandermonde}, which constrains puzzle pieces to preset locations.

\begin{figure}[t]
\centering
\begin{subfigure}[b]{0.4\textwidth}
  \centering
  \includegraphics[scale=0.7]{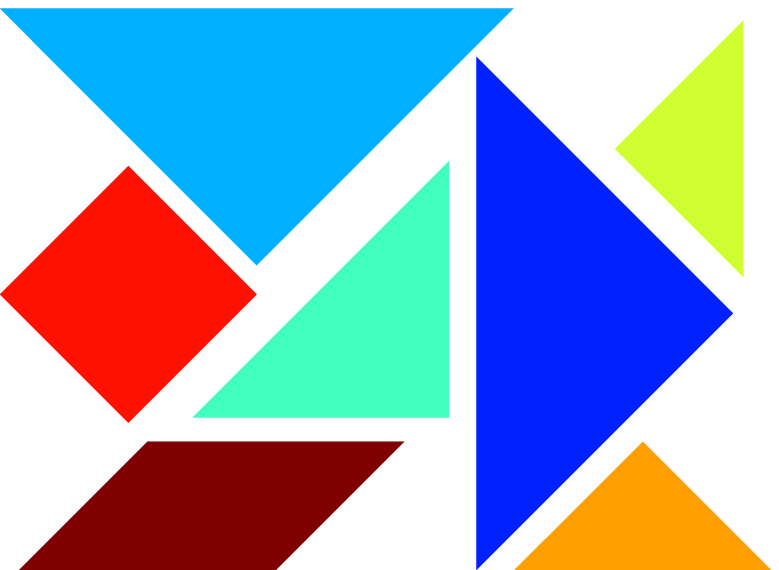}
  \caption{Puzzle pieces}
\end{subfigure}
\begin{subfigure}[b]{0.26\textwidth}
  \centering
  \includegraphics[scale=0.7]{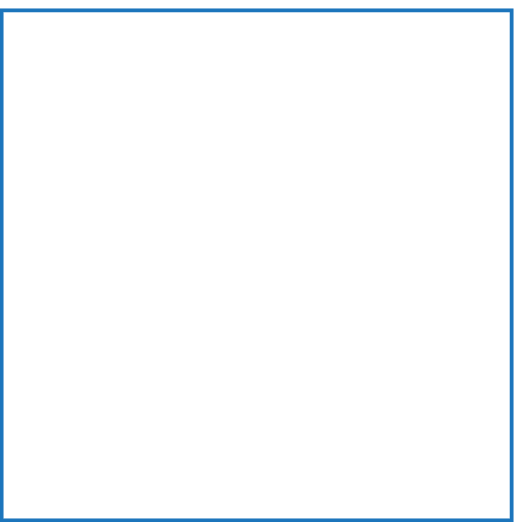}
  \caption{Frame}
\end{subfigure}
\begin{subfigure}[b]{0.26\textwidth}
  \centering
 \includegraphics[scale=0.7]{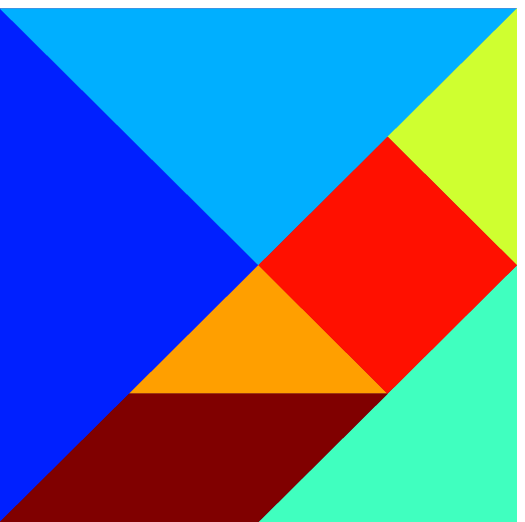}
  \caption{Solution}
\end{subfigure}
\\
\begin{subfigure}[b]{1\columnwidth}
  \centering
  \includegraphics[width=1\columnwidth]{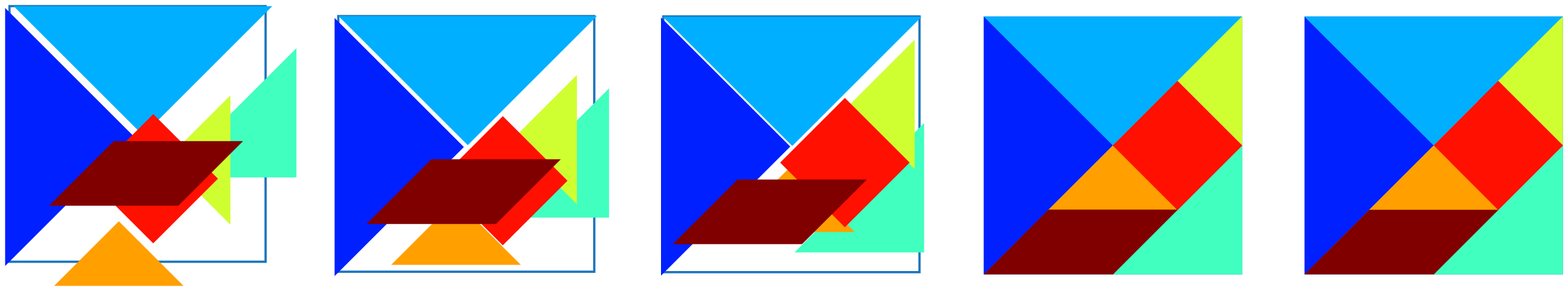}
  \caption{Iterations of \eqref{eqn:eq_equiv_problem_hankel_SDP}, left to right. Converging to a solution of the Tangram.}
\end{subfigure}
\caption{Solution of a Tangram puzzle using our method. In this edge-matching problem all edges share the same color, in the figure pieces are colored for visualization only.}
\label{fig:example_tangram_A}
\end{figure}

\begin{figure}[t]
\centering
\begin{subfigure}[b]{0.28\textwidth}
  \centering
  \includegraphics[scale=0.65]{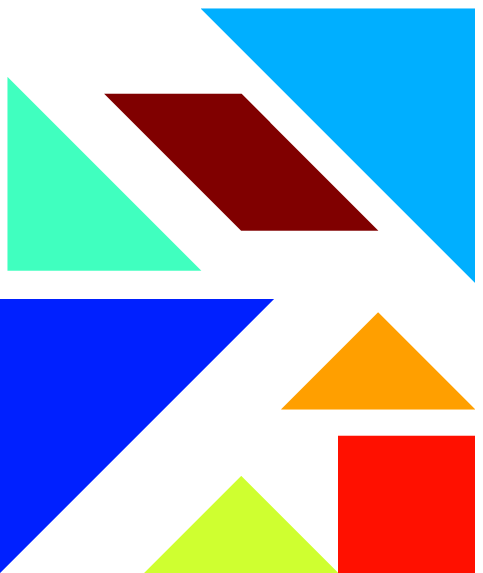}
  \caption{Puzzle pieces}
\end{subfigure}
\begin{subfigure}[b]{0.35\textwidth}
  \centering
  \includegraphics[scale=0.65]{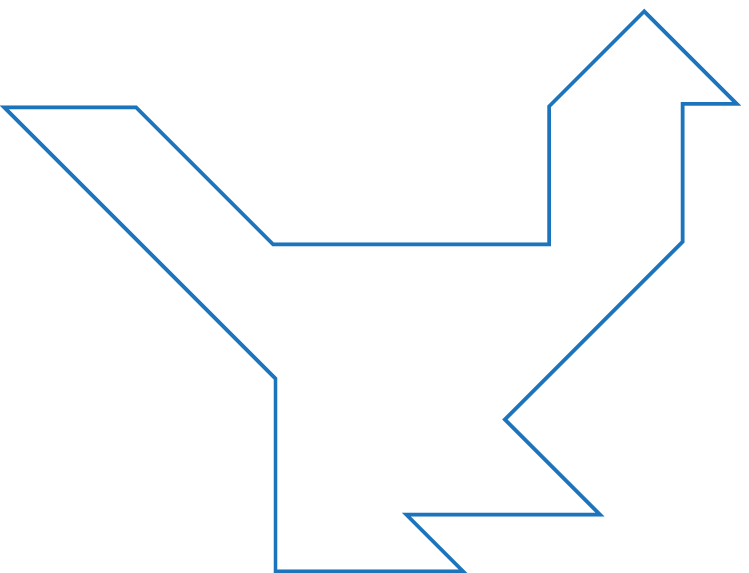}
  \caption{Frame}
\end{subfigure}
\begin{subfigure}[b]{0.35\textwidth}
  \centering
 \includegraphics[scale=0.65]{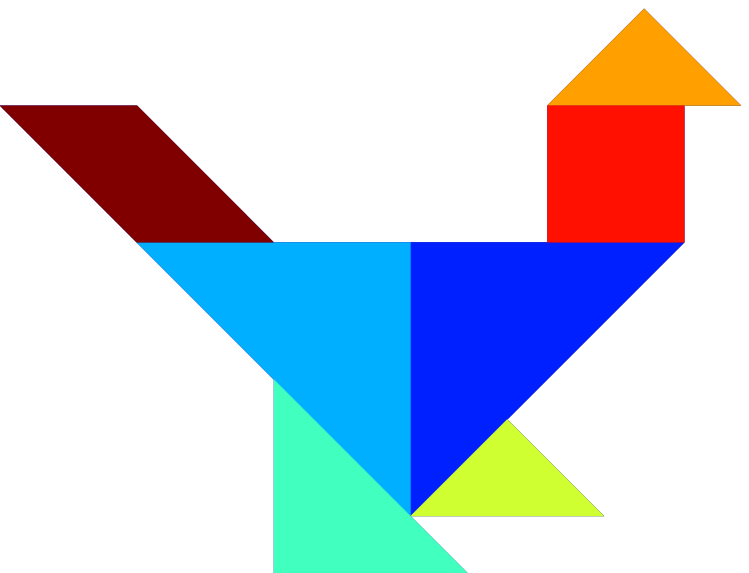}
  \caption{Solution}
\end{subfigure}
\\
\begin{subfigure}[b]{1\columnwidth}
  \centering
 \includegraphics[width=1\columnwidth]{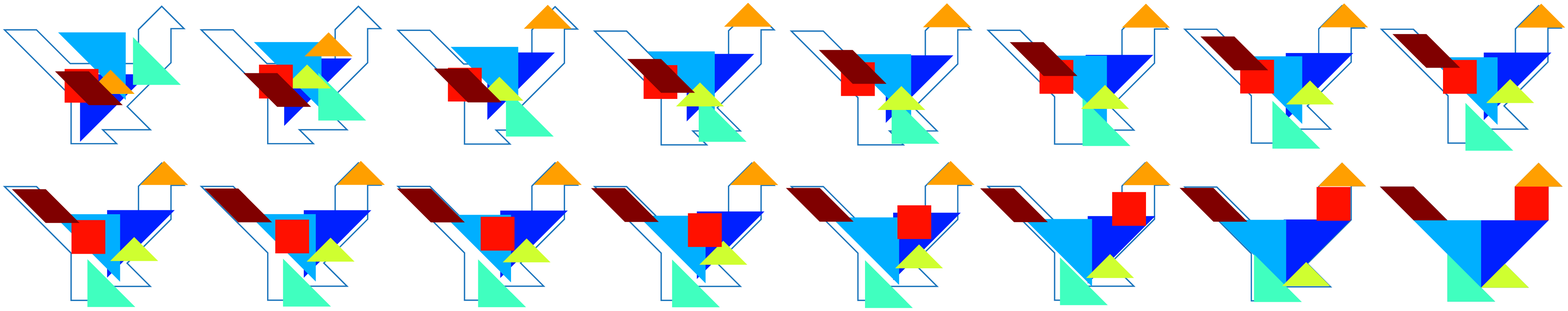}
  \caption{Iterations of \eqref{eqn:eq_equiv_problem_hankel_SDP}, top-left to bottom-right. Converging to a solution of the Tangram.}
\end{subfigure}
\caption{Solution of a Tangram puzzle using our method.}\label{fig:example_tangram_B}
\end{figure}

\subsubsection{Implementation Details and Limitations}
\label{section:implmenetation_details}
Algorithms \ref{alg:solve_approx_lp} and \ref{alg:solve_approx_sdp} were both implemented in Matlab. Yalmip was used for the modeling of the linear program~\eqref{eqn:eq_relax_problem_permutation} and semidefinite program~\eqref{eqn:eq_equiv_problem_hankel_SDP}. Gurobi~\cite{Gurobi} and Mosek~\cite{Mosek}, correspondingly, were used for their optimization.

In theory both LPs and SDPs are solvable in polynomial time. In practice, Algorithm 1 scales moderately well with the number of pieces. We have successfully solved instances of $12\times 12$ puzzles. We haven't had as much success with larger problems, such as the \emph{``Eternity II"} puzzle, which pose a considerable combinatorial challenge. Algorithm 2 has limited scalability, which is dominated by the number and dimension of the positive definite constraints in SDP~\eqref{eqn:eq_equiv_problem_hankel_SDP}. These, in turn, depend on the number of puzzle pieces and edge colors respectively. Current leading SDP optimization engines employ second order interior-point algorithms which limit the applicability of this approach to larger problems.

 \section{Extended Framework} \label{sect:extensions}
In this section we describe extensions of our framework to (i) higher-dimensional puzzles, (ii) general polygonal pieces and (iii) puzzles in which pieces can be rotated to one of a finite set of discrete orientations.

\subsection{High-dimensional Puzzles} \label{sect:extensions_highdim}
Thus far we restricted our attention to 2-dimensional puzzles. We can extend our framework to dimensions $d>2$ as follows.
Consider coordinates $t_i,b_{i,j}\in\Real^d$ and real valued functions $f:\Real^d\to\Real$. Proposition~\ref{proposition:basic_condition} continues to hold, implying as in the 2-dimensional case that if $t_1,\ldots,t_N$ is a solution of the puzzle then
\begin{equation}
\sum_{i,j} s_{i,j}(c,\theta) f\parr{t_i+b_{i,j}} = 0
\end{equation}
for every $(c,\theta)$ and every function $f$.

The polynomial representation proposed in Section~\ref{proposition:equivalent_representation} takes exactly the same form, now by taking $f(u)=\Exp{k^{T}u}$ where $k$ is a $d$-tuple determining the degree of the resulting multivariate polynomial.

The rest of the derivations introduced in the paper follow in the exact same manner with the exception of Proposition~\ref{proposition:equivalent_representation}, whose proof is restricted to the 2-dimensional case (as it relies on identifying $\Real^2$ with $\Complex$.). We conjecture that a similar result holds in arbitrary dimension as well.

\subsection{Shape} \label{sect:extensions_piece_shape}
So far we have assumed puzzle pieces to be \emph{equilateral} polygons. This assumption can be relaxed to the case of general polygons, with arbitrary edge length, by introducing path integrals over the boundary of pieces. When puzzle pieces are not equilateral, Proposition~\ref{proposition:basic_condition} may no longer hold since for example puzzle edges are no longer restricted to pair as a whole with a single other edge. That is, in the solution to the puzzle one edge may be paired with multiple edges or even parts of edges. Nevertheless, an analogous proposition holds provided that \eqref{eqn:eq_basic_condition} is replaced with
\begin{equation}
\sum_{i,j} s_{i,j}(c,\theta) \int_{\gamma_{i,j}} f\parr{t_i+z}dz = 0
\label{eqn:eq_basic_condition_integral}
\end{equation}
where $\gamma_{i,j}$ is the path corresponding to the $j$'th edge of $i$'th piece. The derivations of Sections~\ref{section:probem_statement} and~\ref{section:alg_representation} follow by making the appropriate adaptations. Note that the Tangram examples shown in Figures~\ref{fig:example_tangram_A} and~\ref{fig:example_tangram_B} involve non equilateral polygons. Indeed the algebraic representation which we used to solve them was computed by calculating the path integrals of~\eqref{eqn:eq_basic_condition_integral}.

\subsection{Orientations}
Often one wishes to solve a puzzle for which the orientation of the pieces is also unknown. We now show how the proposed framework can be extended to address this case, under the assumption that finitely many orientations are admissible.

For simplicity, we focus on the 2-dimensional case. Let $\varphi_i$ denote the unknown orientation of the $i$'th piece. Suppose that each orientation $\varphi_i$ belongs to a finite cyclic group of $r$ rotations, namely $\varphi_i \in \set{0,\frac{2\pi }{r},\ldots,\frac{2\pi }{r}(r-1)}$. This is a reasonable assumption for many puzzles (\eg, $90^\circ$ rotations are sufficient for square puzzles, $60^\circ$ rotations for hexagonal puzzles, etc).

Denote by $\R{\theta}$ the rotation matrix
\begin{equation}
\R{\theta} = \begin{bmatrix}\cos(\theta) & \sin(\theta) \\
-\sin(\theta) & \cos(\theta) \\
\end{bmatrix}.
\end{equation}
Next, we prove a claim that allows to establish an algebraic representation in the presence of unknown orientations. The intuitive interpretation is that we duplicate the entire puzzle (frame and pieces) $r$ times, rotated in each of the orientations $\set{0,\frac{2\pi }{r},\ldots,\frac{2\pi }{r}(r-1)}$. Then, all $r$ puzzles are simultaneously solved as a single augmented translation only puzzle, with the location $t_i^{(\rho)}$ of the $\rho$-rotated $i$'th piece linearly constrained by $t_i^{(\rho)}=\R{\frac{2\pi }{r}\rho}t_i$. A solution of this augmented puzzle, under the assumption that pieces may only translate, implies a solution of the original puzzle (with unknown orientations). Namely, a claim analogous to Proposition~\ref{proposition:equivalent_representation} holds, by summing over $r$ rotated copies of the puzzle:

\bigskip
\begin{proposition} \ Suppose that a given puzzle is solved by placing the pieces at the locations $\hat{t}_1,\ldots,\hat{t}_N$ rotated by $\varphi_1,\ldots,\varphi_N$. Then,
\begin{equation}
\sum_{\rho=0}^{r-1}\sum_{i,j} s_{i,j}(c,\theta-\frac{2\pi}{r}\rho) f\parr{\R{\frac{2\pi }{r}\rho}\parr{t_i+b_{i,j}}} = 0
\label{eqn:eq_basic_condition_wrotation}
\end{equation}
for every $(c,\theta)$ and every function $f$, with $t_i=\R{-\varphi_i}\hat{t}_i$.
\label{proposition:basic_condition_rotations}
\end{proposition}

Note that Equation~\ref{eqn:eq_basic_condition_wrotation} yields a representation of the puzzle which is invariant to rotations of the pieces. The solution will consist of $r$ copies of the puzzle. Each piece is associated with a possibly different rotation angle according to the copy of the puzzle to which it is translated as illustrated in Figure~\ref{fig:example_rotations}.

\begin{proof}
Proposition~\ref{proposition:equivalent_representation} implies that if the puzzle is solved by placing the pieces at the locations $\hat{t}_1,\ldots,\hat{t}_N$ rotated by $\varphi_1,\ldots,\varphi_N$ then for every $(c,\theta)$ and every function $f$ we have
 \begin{equation}
 \label{eqn:basic_eqn_wrotation1}
 \sum_{i,j} s_{i,j}(c,\theta-\varphi_i) f\parr{\hat{t}_i+\R{\varphi_i}b_{i,j}} = 0.
 \end{equation}
 Set $t_i=\R{-\varphi_i}\hat{t}_i$ to yield
 \begin{equation}
 \label{eqn:basic_eqn_wrotation2}
 \sum_{i,j} s_{i,j}(c,\theta-\varphi_i) f\parr{\R{\varphi_i}\parr{t_i+b_{i,j}}} = 0.
 \end{equation}
 Next, sum over all admissible rotations
 \begin{equation}
 \label{eqn:basic_eqn_wrotation3}
\sum_{\rho=0}^{r-1} \sum_{i,j} s_{i,j}(c,\theta-\varphi_i-\frac{2\pi}{r}\rho) f\parr{\R{\varphi_i+\frac{2\pi}{r}\rho}\parr{t_i+b_{i,j}}} = 0.
 \end{equation}
This equation can be understood as the summation over $r$ rotated copies of the puzzle (as illustrated in Figure~\ref{fig:example_rotations}). The proposition follows by noticing that since $\varphi_i \in \set{0,\frac{2\pi }{r},\ldots,\frac{2\pi }{r}(r-1)}$, all terms $\varphi_i$ may be omitted.
\hfill
\end{proof}
\bigskip

Proposition~\ref{proposition:basic_condition_rotations} provides a construction which enables us to apply the algorithms developed in Sections~\ref{section:alg_representation} and~\ref{section:solving} to solve puzzles with an unknown discrete set of orientations. In Figure~\ref{fig:example_rotations} we show an example where we use this construction to solve a puzzle where pieces can be rotated into one of four different orientations.

\begin{figure}[t]
\centering
\begin{subfigure}[b]{0.3\textwidth}
  \centering
  \includegraphics[scale=0.9]{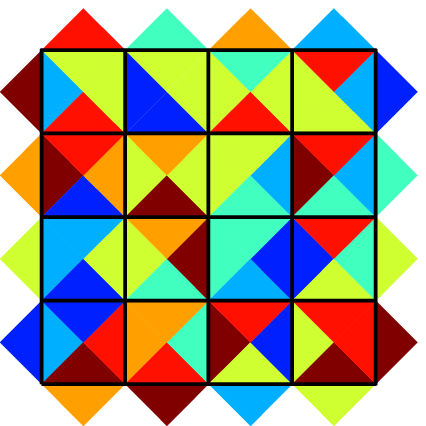}
  \caption{Scrambled}
  \includegraphics[scale=0.9]{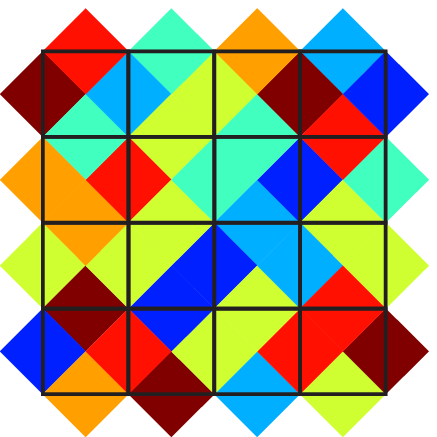}
  \caption{Solution}
\end{subfigure}
\begin{subfigure}[b]{0.69\columnwidth}
  \centering
  \includegraphics[scale=0.7]{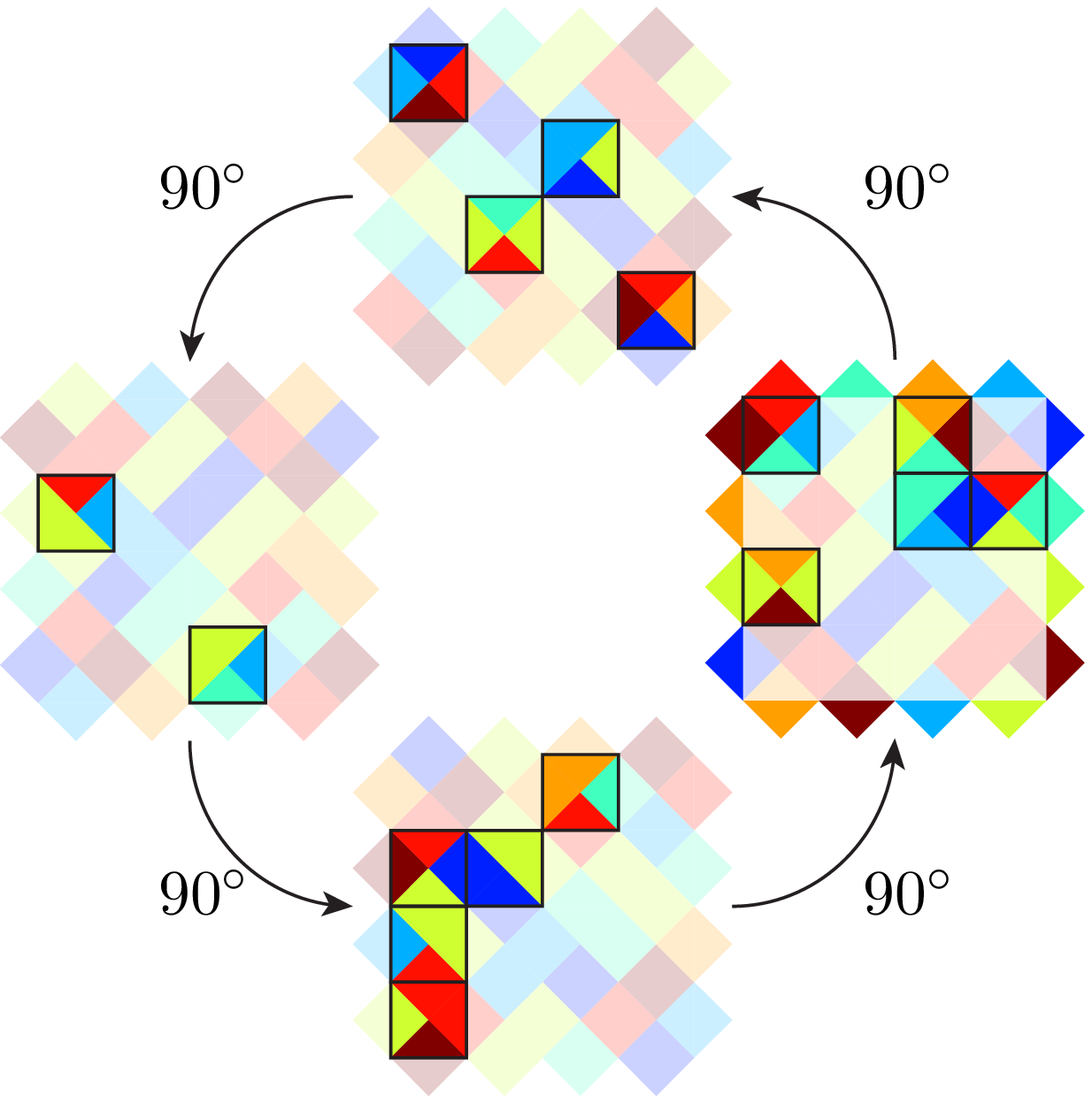}
  \caption{Solution of the augmented puzzle}
\end{subfigure}
\caption{Unknown orientation addressed by solving an augmented puzzle, generated by duplicating the entire puzzle, rotated in each of the admissible orientations. Rotated pieces (and frames) are visualized by transparency.}\label{fig:example_rotations}
\end{figure}

\section{Concluding Remarks}
In this paper, we propose a novel representation for apictorial edge-matching puzzles. We explain how to generate systems of polynomial equations which are satisfied by puzzle solutions. We further show how to construct systems which are \emph{complete} representations for 2-dimensional translation only puzzles. We prove that for these systems the converse also holds, \ie, any solution of the system is a solution of the puzzle.

Using this representation, we devise two algorithms for approximating solutions of edge-matching puzzles. Both algorithms rely on solving a sequence of \emph{continuous} and \emph{global} convex relaxations. An iterative algorithm based on linear programming relaxation is proposed for the case where we know in advance that pieces may only be placed in a finite set of predetermined locations. For the more general setting, we propose an iterative algorithm based on semidefinite programming relaxation.

Finally, we extend our computational framework and algorithms to various interesting variants of edge-matching puzzles including higher-dimensional puzzles, puzzles in which the pieces have irregular shapes and puzzles in which pieces can be rotated to one of a finite set of discrete orientations.

\bigskip
\paragraph{Acknowledgements}
Research was supported in part by the Israel Science Foundation grants numbers 764/10 and 1265/14 and by the Minerva Foundation. The vision group at the Weizmann Institute is supported in part by the Moross Laboratory for Vision Research and Robotics.

\bigskip
\appendix
\section{Proof of Proposition~\ref{proposition:equivalent_representation} (Complete Polynomial Representation)}
\label{section:proof_equivalence}

\bigskip
\mbox{\sffamily{\color{header1}Proposition~\ref{proposition:equivalent_representation}.}}
{\it
There exists (constructively) a polynomial system $\PP$ with
\begin{equation*}
K(c,\theta) = \#\set{\text{edges of type } (c,\theta)}
\end{equation*}
equations for each edge type $(c,\theta)$ that satisfies that $T_1,\ldots,T_N$ is a solution of $\PP$ if and only if it is a solution of the puzzle.}
\bigskip

\begin{proof}
We prove the proposition in three main steps: we (i) consider each edge type $(c,\theta)$ separately and construct a polynomial system $\PP_1$ that determines all admissible edge pairings; (ii) add constraints accounting for the geometry of puzzle pieces, thus achieving a polynomial system $\PP_2$ which exactly determines the puzzle solutions; and (iii) show that $\PP_3$, the polynomial representation of Proposition~\ref{proposition:equivalent_representation}, is equivalent to $\PP_2$, and is therefore a complete representation of the puzzle.

\bigskip
To prove the proposition we shall use complex numbers to represent 2-dimensional coordinates, that is, we consider piece locations $t_1,\ldots,t_N\in\Complex$.
This will both simplify notations and facilitate our derivations.

We begin by proving a simple lemma. Fix $v_1,\ldots,v_N \in \Complex$ and consider the polynomial system
\begin{equation}
\set{p_k(u_1,\ldots,u_N)=\sum_i u_i^k - \sum_iv_i^k = 0}_{k=1,\ldots,K}
\label{eq:eq_combinatorical_system}
\end{equation}
in the variables $u_1,\ldots,u_N$.

\bigskip
\begin{lemma} \label{lemma:bezout}
For $K\geq N$, the polynomial system \eqref{eq:eq_combinatorical_system} has exactly $N!$ solutions:
\begin{equation*}
(u_1,\ldots,u_N) = (v_{\sigma(1)},\ldots,v_{\sigma(N)})
\end{equation*}
for all permutations $\sigma$ of $1,\ldots,N$.
\end{lemma}
\begin{proof} Denote by $S$ the set of solutions of \eqref{eq:eq_combinatorical_system}. By construction, $(v_{\sigma(1)},\ldots,v_{\sigma(N)}) \in S$ for any permutation $\sigma$, therefore, $|S|\geq N!$.

On the other hand, if we assume that $K=N$ then B\'ezout's theorem \cite{Coolidge2004, Sottile2002} asserts that
\begin{equation*}
|S| \leq \prod_{i=1}^N deg\{p_i\} = \prod_{i=1}^N i = N!.
\end{equation*}
This bound clearly holds for $K>N$, thus concluding the proof.
\hfill
\end{proof}
\bigskip

We shall refer to a system of the form \eqref{eq:eq_combinatorical_system} which satisfies the conditions of Lemma~\ref{lemma:bezout} as a \emph{bipartite system}, since it encodes all possible perfect matchings in the complete bipartite graph $\mathbf{K}_{N,N}$ whose vertices are the $u_i$'s and $v_i$'s, see Figure~\ref{fig:bipartite_match}.

\begin{figure}[h!]
  \centering
  \includegraphics[]{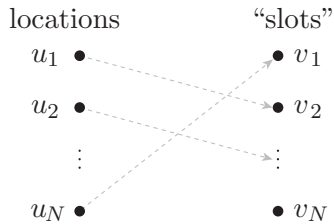}
  \caption{Discrete piece location assignment as a bipartite matching problem.}
  \label{fig:bipartite_match}
\end{figure}

\bigskip
Consider the global coordinate $t_{i,j}=t_i+b_{i,j}$ of every edge element \emph{independently}, correspondingly, consider the exponential coordinates
\begin{equation}
T_{i,j} = \Exp{t_{i,j}} = \Exp{t_i+b_{i,j}} = \Exp{t_i} \Exp{b_{i,j}} = T_iB_{i,j}
\end{equation}

We first treat each edge type $(c,\theta)$ separately. Let
\begin{equation}
V(c,\theta) = \set{(i,j):c_{i,j}=c \text{ and } \theta_{i,j}=\theta}
\end{equation}
be the set of indices of $(c,\theta)$-edges, and
\begin{equation}
\V(c,\theta) = \set{(i,j):c_{i,j}=c \text{ and } \theta_{i,j}=\theta+\pi}
\end{equation}
be the set of indices of their complimentary $(c,\theta+\pi)$-edges. Consider the set of polynomial equations
\begin{equation}
\PP_{(c,\theta)} = \set{ \sum_{(i,j)\in V(c,\theta)} {T_{i,j}}^k = \sum_{(i,j)\in \V(c,\theta)} {T_{i,j}}^k }_{k=1,\ldots,K(c,\theta)}
\label{eqn:Pctheta_Tij}
\end{equation}
of degrees $1,2,\ldots,K(c,\theta)$.

The sets $V(c,\theta)$ and $\V(c,\theta)$ are disjoint and $|V(c,\theta)|=|\V(c,\theta)|$, otherwise the puzzle is unsolvable. As such, we may consider the RHS of \eqref{eqn:Pctheta_Tij} fixed and employ Lemma~\ref{lemma:bezout}.

Namely, for sufficiently many polynomials, specifically,
\begin{equation}
K(c,\theta)\geq |V(c,\theta)|=\#\set{\text{edges of type } (c,\theta)}
\label{eqn:eq_inequality_Kctheta_geq_V}
\end{equation}
the polynomial system \eqref{eqn:Pctheta_Tij} is a bipartite system, hence solutions of \eqref{eqn:Pctheta_Tij} are exactly the valid pairings of edges in $V(c,\theta)$ and $\V(c,\theta)$.

Next, we take the union of all such systems, over all types of edges, with $K(c,\theta)$ chosen according to \eqref{eqn:eq_inequality_Kctheta_geq_V}:
\begin{equation}
\PP_1 = \bigcup_{(c,\theta)} \PP_{(c,\theta)}.
\end{equation}
This is a polynomial system of equations, in variables $T_{i,j}$, in particular \emph{every edge} of the puzzle is constrained to satisfy the bipartite relation of the system it belongs to. That is, all edge elements of the puzzle are paired with appropriate edges.

However, employing independent coordinates for edge elements $T_{i,j}$ is insufficient. Since the internal geometry of a puzzle piece is not taken into account, a solution of $\PP_1$ may imply an invalid puzzle assignment (\eg, two edge elements of the same piece can pair).

This, however, may be easily resolved by relating the coordinates of an edge to the piece it belongs to, namely, by adding the constraints $T_{i,j} = T_i B_{i,j}$ for all pieces. Therefore, generating the augmented system
\begin{equation}
\PP_2 = \PP_1 \cup \set{T_{i,j} = T_i B_{i,j}}_{i,j}.
\end{equation}
The solutions of $\PP_2$ must therefore satisfy both the bipartite relations and the internal geometry of each piece, that is, a solution of $\PP_2$ is a solution of the puzzle.

Lastly, we directly employ these linear constraints and substitute $T_{i,j} = T_i B_{i,j}$ into \eqref{eqn:Pctheta_Tij} to obtain the polynomial system
\begin{equation}
\widehat{\PP}_{(c,\theta)} = \set{ \sum_{(i,j)\in V(c,\theta)} {T_i}^k {B_{i,j}}^k = \sum_{(i,j)\in \V(c,\theta)} {T_i}^k {B_{i,j}}^k }_{k=1,\ldots,K(c,\theta)}.
\label{eqn:Pctheta_TiBij}
\end{equation}
Using the notation $\alpha_i^{(k,c,\theta)}=\sum_j {s_{i,j}(c,\theta) B_{i,j}^k}$ of Section~\ref{section:poly_representation}, this can be equivalently written as
\begin{equation}
\widehat{\PP}_{(c,\theta)} = \set{\sum_{i} \alpha_i^{(k,c,\theta)}T_i^k = 0 }_{k=1,\ldots,K(c,\theta)}.
\label{eqn:Pctheta_AlphaiTi}
\end{equation}
Again, taking the union over all types of edges yield the polynomial system
\begin{equation}
\PP_3 = \bigcup_{(c,\theta)} \widehat{\PP}_{(c,\theta)}
\end{equation}
with $K(c,\theta)$'s chosen as before.

Notice that by construction the polynomial systems $\PP_3$ and $\PP_2$ are equivalent. Therefore, it is guaranteed that a solution of $\PP_3$ corresponds to a valid solution of the puzzle.
\hfill
\end{proof}

\section{Solution of 2-dimensional puzzles}
\label{section:solution_2d_techincal}

The presentation in Section~\ref{section:solving} was simplified for notational purposes by considering 1-dimensional puzzles (\ie, coordinates over $\Real$).  We now provide the technical details required for adapting these results to the case of 2-dimensional puzzles. Similar modifications can be applied to address puzzles of arbitrary dimension, thereby completing the discussion of Section~\ref{sect:extensions_highdim} regarding puzzles in higher dimensions.

In Section~\ref{section:reformulation_vandermonde} the requirement that $X$ is a Vandermonde matrix needs to be replaced with a multivariate analogue. Recall that in the 2-dimensional case the variable associated with the $i$'th piece is a 2-vector $T_i=[T_{ix}\ T_{iy}]^{T}$. Correspondingly, $X$ in the optimization problem \eqref{eqn:eq_equiv_problem_vandermonde} should be of the form
\begin{equation}
\begin{bmatrix}
1 & T_{1x} & T_{1y} & T_{1x}^2 & T_{1x}T_{1y} & T_{1y}^2 & \cdots & T_{1y}^K\\
1 & T_{2x} & T_{2y} & T_{2x}^2 & T_{2x}T_{2y} & T_{2y}^2 & \cdots & T_{2y}^K\\
1 & T_{3x} & T_{3y} & T_{3x}^2 & T_{3x}T_{3y} & T_{3y}^2 & \cdots & T_{3y}^K\\
\vdots & \vdots & \vdots & \ddots & \vdots \\
1 & T_{Nx} & T_{Ny} & T_{Nx}^2 & T_{Nx}T_{Ny} & T_{Ny}^2 & \cdots & T_{Ny}^K\\
\end{bmatrix}.
\label{eqn:vandermonde_2d_matrix}
\end{equation}
That is, a matrix whose $i$'th row entries are all the monomials in $T_i$ of total degree up to $K$.

Similarly, in Section~\ref{section:reformulation_rank_one} the requirement that each matrix $Z_i$ is a Hankel matrix is replaced by constraining $Z_i$ to be of the form corresponding to the outer product of all univariate monomials in the entries of $T_i$ of degree up to $K$, namely, each $Z_i$ should be of the form
\begin{equation}
\begin{pmat}[{}]
1 \cr\- T_{ix} \cr \vdots \cr T_{ix}^K \cr\-
     T_{iy} \cr \vdots \cr T_{iy}^K \cr
\end{pmat}
\begin{pmat}[{}]
1 \cr\- T_{ix} \cr \vdots \cr T_{ix}^K \cr\-
     T_{iy} \cr \vdots \cr T_{iy}^K \cr
\end{pmat}^T
=
\begin{pmat}[{|..|}]
1 & T_{ix} & \cdots & T_{ix}^K &
    T_{iy} & \cdots & T_{iy}^K \cr\-
T_{ix} & T_{ix}^2 & & T_{ix}^{K+1} &
T_{ix}T_{iy} & \cdots & T_{ix}T_{iy}^K \cr
\vdots &  &  \iddots & & \vdots &  &  \vdots \cr
T_{ix}^{K} & T_{ix}^{K+1} & \cdots  & T_{ix}^{2K} &
T_{ix}^KT_{iy} & \cdots & T_{ix}^KT_{iy}^K \cr\-
T_{iy} & T_{ix}T_{iy} & & T_{ix}^KT_{iy} &
T_{iy}^2 & \cdots & T_{iy}^{K+1} \cr
\vdots &  &  \iddots & & \vdots &  &  \vdots \cr
T_{iy}^{K} & T_{ix}T_{iy}^K & \cdots  & T_{ix}^KT_{iy}^K &
T_{iy}^{K+1} & \cdots & T_{iy}^{2K} \cr
\end{pmat}.
\label{eqn:eq_hankel_matrix_2d_Ts}
\end{equation}
Note that the main diagonal blocks of $Z_i$ are now Hankel matrices, however, the entire matrix $Z_i$ is no longer Hankel.

\bibliographystyle{siam}

\end{document}